\documentclass[10pt,twoside]{article}

\usepackage[english]{babel}
\usepackage[utf8]{inputenc}
\usepackage{amsmath}
\usepackage{amsfonts}
\usepackage{mathtools}
\usepackage{amssymb}
\usepackage{amsthm}
\usepackage{mathrsfs}%
\usepackage{hyperref}
\usepackage{graphicx}
\usepackage{epstopdf}
\usepackage{todonotes}

\def\R{{\mathbb R}}

\def\N{\mathbb N}

\def\sC{\mathscr{C}}

\def\cH{{\mathcal H}}
\def\cO{{\mathcal O}}
\def\cS{{\mathcal S}}

\def\st{\, {\bf :} \,}
\def\pd{{\bf PD}}

\newtheorem{theorem}{Theorem}

\newtheorem{lemma}[theorem]{Lemma}
\newtheorem{proposition}[theorem]{Proposition}
\newtheorem{corollary}[theorem]{Corollary}

\newtheorem{definition}[theorem]{Definition}

\newtheorem{remark}[theorem]{Remark}

\begin{document}

%#########################################################################
%### Page design
%#########################################################################

\pagestyle{myheadings}

%#########################################################################
%### Title page
%#########################################################################

\title{On the Convergence of Irregular Sampling in Reproducing Kernel Hilbert Spaces}

\date{\today}

\author{\large
Armin Iske\footnote{Universit\"at Hamburg, Department of Mathematics, {\tt armin.iske@uni-hamburg.de}}}

\markboth{\footnotesize \rm \hfill ARMIN ISKE\hfill}
{\footnotesize \rm \hfill CONVERGENCE IN REPRODUCING KERNEL HILBERT SPACES\hfill}

\maketitle
\thispagestyle{plain}

%%%%%%%%%%%%%%%%%%%%%%%%%%%%%%%%%%%
% A B S T R A C T
%%%%%%%%%%%%%%%%%%%%%%%%%%%%%%%%%%%

\begin{abstract}
We analyse the convergence of sampling algorithms for functions in reproducing kernel Hilbert spaces (RKHS).
To this end, we discuss approximation properties of kernel regression under minimalistic assumptions on both the kernel and the input data.
We first prove error estimates in the kernel's RKHS norm. 
This leads us to new results concerning uniform convergence of kernel regression on compact domains.
For Lipschitz continuous and H\"older continuous kernels, we prove convergence rates.
\end{abstract}

%%%%%%%%%%%%%%%%%%%%%%%%%%%%%%%%%%%%%%%%%%%%%%%%%%%%%%%%%%%%%%%%
% I N T R O D U C T I O N
%%%%%%%%%%%%%%%%%%%%%%%%%%%%%%%%%%%%%%%%%%%%%%%%%%%%%%%%%%%%%%%%
%------------------------------------------------------------------------------------------
\section{Introduction}
%------------------------------------------------------------------------------------------
Learning theory~\cite{Cucker2001,Cucker2007} requires the approximation of functions $f : \Omega \longrightarrow \R$ from irregular samples of $f$ on a compact domain $\Omega \subset \R^d$, for $d>1$.
From the viewpoint of statistical learning theory~\cite{Vapnik1998}, the general purpose of learning is referred to {\em data regression} (or {\em data fitting}),
where the basic task is to determine a regression function $g : \Omega \longrightarrow \R$ from samples $f_X= \{ f(x) \}_{x \in X}$ taken at sampling points $X \subset \Omega$. 

\medskip
Reproducing kernels provide popular concepts for data regression in machine learning~\cite{Schoelkopf2002}, in particular for support vector machines~\cite{Steinwart2008}.
In this case, the target $f$ is assumed to lie in a Hilbert space $\cH_{K,\Omega}$ of functions, being generated by a (conditionally) positive definite kernel function $K$ on $\Omega$.

\medskip
The theory on {\em reproducing kernel Hilbert spaces} (RKHS) is dating back to the seminal work~\cite{Aronszajn1950} of Aronszajn (in 1950).
Contemporary questions on kernel-based learning are concerning approximation properties of kernel regression~\cite{Cucker2007,Smale2003,Zhou2003}.
Quite recently, dimensionality reduction in kernel regression has been investigated in~\cite{Fukumizu2009} from the viewpoint of statistics. 
Moreover, convergence rates and stability results for a general high-dimensional kernel regression framework were proven in~\cite{EIT2023}, 
where rather specific assumptions on both the kernel $K$ and the sampling points $X$ were essentially needed.

\medskip
In this work, we analyse the convergence of kernel regression in RKHS under {\em minimalistic} 
assumptions on the kernel $K$, and so on the RKHS $\cH_{K,\Omega}$, and on the sampling points~$X$.
 
\medskip
The outline of this paper is as follows.
We first explain key features on RKHS (in Section~\ref{sec:regression}) and 
on kernel regression (in Section~\ref{sec:problem}).
Then, we formulate minimalistic assumptions for $K$ and $X$ (in Section~\ref{sec:minimalistic}),
under which we can prove convergence of kernel regression (in Section~\ref{sec:convergence})
with respect to the kernel's RKHS norm and for uniform convergence.

%%=============================================%%
\section{Three Key Features of Kernel Regression}
\label{sec:regression}
%%=============================================%%
%
Starting point for our discussion on kernel regression are positive definite functions
(for details refer to~\cite{Buhmann2003,Iske2018,Wendland2005}).

\begin{definition}
\label{bas:def:pdf}
For $\Omega \subset \R^d$, we say that a continuous and symmetric function $K : \Omega \times \Omega \longrightarrow \R$ 
is a {\em positive definite kernel} on $\Omega$, $K \in \pd (\Omega)$, 
if for any finite set of pairwise distinct points $X = \{x_1,\ldots,x_n\} \subset \Omega$, $n \in \N$, 
the matrix
\begin{equation*}
%\label{bas:equ:pdm}
    A_{K,X} = (K(x_k,x_j))_{1 \leq j,k \leq n} \in \R^{n \times n}
\end{equation*}
is symmetric and positive definite.
\end{definition}

Positive definite kernels on $\R^d$ are often required to be {\em translation invariant}, i.e.,  
$K$ is assumed to have the form
\begin{equation}
\label{bas:equ:tri}
    K(x,y) = \Phi(x - y) 
    \qquad \mbox{ for } x,y \in \R^d
\end{equation}
for an even function $\Phi : \R^d \longrightarrow \R$.
Popular examples for translation invariant kernels $K \in \pd(\R^d)$ include the {\em Gaussian} 
$
  K(x,y) = \Phi(x-y) = \exp(- \| x-y \|_2^2),
$
and the {\em inverse multiquadric}
$
  K(x,y) = \Phi(x-y) = ( 1 + \| x - y \|_2^2 )^{-1/2},
$
where $\|\cdot\|_2$ denotes as usual the Euclidean norm on $\R^d$.

\medskip
Next we explain the basic setup of kernel regression in learning theory~\cite{Zhou2003}.
To this end, for fixed domain $\Omega \subset \R^d$, let
$K : \Omega \times \Omega \longrightarrow \R$ be positive definite on $\Omega$, i.e., $K \in \pd(\Omega)$. 
In the following discussion, it will be convenient to let the function $K_x : \Omega \longrightarrow \R$, for $x \in \Omega$, be defined as
$$
    K_x(y) := K(x,y)
    \qquad \mbox{ for } x,y \in \Omega.
$$
Then, according to the seminal work of Aronszajn~\cite{Aronszajn1950},
the {\em reproducing kernel Hilbert space} (RKHS) $\cH_{K,\Omega}$ associated with $K \in \pd(\Omega)$ is the closure
$$
    \cH_{K,\Omega} := \overline{ {\rm span} \left\{ K_x \st x \in \Omega \right\} }
$$
with respect to the inner product $(\cdot,\cdot)_K \equiv (\cdot,\cdot)_{\cH_{K,\Omega}}$ satisfying
$$
    (K_x,K_y)_K = K(x,y)
    \qquad \mbox{ for all } x,y \in \R^d,
$$
whereby we have
\begin{eqnarray*}
    \left\| \sum_{j=1}^n c_j K_{x_j} \right\|_K^2 & := & 
    \left( \sum_{j=1}^n c_j K_{x_j} , \sum_{k=1}^n c_k K_{x_k} \right)_K 
    \\ &=& 
    \sum_{j,k=1}^n c_j K(x_j,x_k) c_k
    =
    c^\top  A_{K,X} c,
\end{eqnarray*}
for all $X = \{x_1,\ldots,x_n\} \subset \Omega$ and $c=(c_1,\ldots,c_n)^\top \in \R^n$.

\bigskip
Now let us recall three key features of kernel regression.

\bigskip
{\bf Feature~1:} 
The {\em reproducing kernel property}
\begin{equation}
\label{reg:equ:rep}
    f(x) = ( K_x , f )_K
    \qquad \mbox{ for all } x \in \Omega
\end{equation}
holds for all $f \in \cH_{K,\Omega}$. 
In particular, for any $K_x \in \cH_{K,\Omega}$,
\begin{equation}
\label{reg:equ:repK}
    K_x(y) = ( K_y , K_x )_K = K(x,y)
    \mbox{ for all } x,y \in \Omega,
\end{equation}
holds, whereby we have $K_x(y)=K_y(x)$, for all $x,y \in \Omega$.

The reproducing kernel properties~(\ref{reg:equ:rep}) and~(\ref{reg:equ:repK}) lead us to
\begin{eqnarray*}
    | f(x) - f(y) |^2 &=& | ( K_x - K_y, f )_K |^2 
%    \\ &\leq& 
\leq
    \| K_x - K_y \|^2_K \cdot \| f \|^2_K 
    \\ &=& 
    (K_x(x) - 2 K_x(y) + K_y(y)) \cdot \| f \|^2_K,
\end{eqnarray*}
which immediately implies the continuity of $f \in \cH_{K,\Omega}$, due to the continuity of $K$ on $\Omega \times \Omega$.
Therefore, the reproducing kernel Hilbert space $\cH_{K,\Omega}$ of $K$ 
is embedded in the continuous functions on $\Omega$, i.e., $\cH_{K,\Omega} \subset \sC(\Omega)$.

\bigskip
{\bf Feature~2:} 
For $X = \{x_1,\ldots,x_n\} \subset \Omega$ we let 
$$
    \cS_{K,X} := {\rm span} \left\{ K_x \st x \in X \right\} \subset \cH_{K,\Omega} 
$$
denote the $n$-dimensional subspace of $\cH_{K,\Omega}$ spanned by $X$.
Then, the {\em orthogonal projection} of $f \in \cH_{K,\Omega}$ onto $\cS_{K,X}$ is the unique interpolant $s_{f,X} \in \cS_{K,X}$ to $f$ on $X$.
In other words, the interpolant $s_{f,X}$ to $f$ on $X$ is the unique best approximation to $f$ with respect to the RKHS norm $\| \cdot \|_K = (\cdot,\cdot)^{1/2}_K$,
i.e.,
$$
    \| s_{f,X} - f \|_K \leq \| s - f \|_K 
    \qquad \mbox{ for all } s \in \cS_{K,X}.
$$
In conclusion, the interpolant $s_{f,X}$ is the best regression fit to $f \in \cH_{K,\Omega}$ from data $f_X = (f(x_1),\ldots,f(x_n))^\top \in \R^n$.
Moreover, $s_{f,X} \in \cS_{K,X}$ has the form
$$
    s_{f,X} = \sum_{j=1}^n c_j K_{x_j} 
$$
where the coefficient vector $c=(c_1,\ldots,c_n)^\top \in \R^n$ is the unique solution of the linear equation system $A_{K,X} c = f_X$,
due to the interpolation conditions $s_{f,X}(x_k)=f(x_k)$, for all $1 \leq k \leq n$. 

\bigskip
{\bf Feature~3:} 
The orthogonality $s_{f,X} - f \perp \cS_{K,X}$ implies
\begin{equation}
\label{reg:equ:nmin}
    \| s_{f,X} \|_K \leq \| f \|_K
    \quad \mbox{ and } \quad
    \| s_{f,X} - f \|_K \leq \| f \|_K,
\end{equation}
due to the {\em Pythagoras theorem}
$$
     \| f \|^2_K =  \| s_{f,X} - f \|^2_K + \| s_{f,X} \|^2_K.
$$
In other words, the kernel regression $s_{f,X} \in \cS_{K,X}$ minimizes the RKHS norm $\|\cdot\|_K$ 
among all interpolants to the samples $f_X$ from $\cH_{K,\Omega}$.
Therefore, kernel regression can be viewed as a spline approximation method.

%%=============================================%%
\section{Problem Formulation and Further Notations}
\label{sec:problem}
%%=============================================%%

Let $X = (x_k)_{k \in \N}$ be a sequence of pairwise distinct points in $\Omega$.
We use the notation $X_n = \{ x_1,\ldots,x_n \} \subset \Omega$ for the (ordered) point set containing the first $n$ points in $X$.

Recall that each point set $X_n \subset \Omega$ 
spans a finite dimensional regression space $\cS_{K,X_n}$. 
Moreover, recall that for any target $f \in \cH_{K,\Omega}$ there is one unique minimizer 
$s_{f,X_n} \in \cS_{K,X_n}$ of the kernel regression error
\begin{equation}
\label{bas:equ:err}
    \eta_n \equiv \eta_n(f,\cS_{K,X_n}) := \| s_{f,X_n} - f \|_K
    \qquad \mbox{ for } n \in \N.
\end{equation}
For notational brevity, we let 
$
    s_n := s_{f,X_n},
$
for $n \in \N$.

\medskip
{\bf Problem Formulation:} 
We analyze the convergence of kernel regression under minimalistic assumptions.
To be more precise, we prove convergence results of the form
\begin{equation}
\label{bas:equ:con}
    \| s_n - f \| \longrightarrow 0 \quad 
    \mbox{ for } n \to \infty
\end{equation}
under mild as possible conditions on the kernel $K \in \pd(\Omega)$, on the target $f \in \cH_{K,\Omega}$ and on the sample points $X = (x_k)_{k \in \N}$.
Our convergence analysis is first done with respect to the RKHS norm $\| \cdot \|_K$, before we turn to uniform convergence.
For the case of uniform convergence, we prove convergence rates under slightly more restrictive conditions on $K \in \pd(\Omega)$.

\medskip
In our analysis, the sequence $(h_n)_{n \in \N}$ of {\em fill distances}
\begin{equation}
\label{bas:equ:fil}
    h_n \equiv h(X_n,\Omega) := \sup_{ y \in \Omega } \, \min_{x \in X_n} \| y - x \|_2
    \qquad \mbox{ for }  n \in \N
\end{equation}
of $X_n$ in $\Omega$ will play an important role. Note that the (non-negative) fill distances 
$(h_n)_{n \in \N}$ of the sequence $X = (x_k)_{k \in \N}$ are monotonically decreasing. 
We remark already at this point that we can only obtain convergence in~\eqref{bas:equ:con}, 
if $(h_n)_{n \in \N}$ is a zero sequence, i.e., if $h_n \searrow 0$ for $n \to \infty$.

%%=============================================%%
\section{Minimalistic Assumptions}
\label{sec:minimalistic}
%%=============================================%%

%%=============================================%%
\subsection{Minimalistic Assumptions on the Kernel}
\label{sub:min:kernel}
%%=============================================%%
We remark that the required continuity of $K \in \pd (\Omega)$, as stated at the outset of this work, 
is necessary for the well-posedness of kernel regression on (truly multi-dimensional) domains $\Omega$.
This is due to the classical theorem of Mairhuber-Curtis from approximation theory, according to which 
there are no non-trivial Haar systems on domains $\Omega \subset \R^d$, for $d>1$, containing bifurcations (cf.~\cite[Theorem~5.25]{Iske2018}).

To prove convergence of kernel regression with respect to $\|\cdot\|_K$, 
we won't require any further (stricter) assumptions on $K \in \pd (\Omega)$ other than its continuity on $\Omega \times \Omega$.
Moreover, we won't require any conditions on $\Omega$.
To prove decay rates for uniform convergence, we will merely require local H\"older continuity for $K \in \pd(\Omega)$, cf.~Definition~\ref{bas:def:holK}.

%%=============================================%%
\subsection{Minimalistic Assumptions on the Target Functions}
\label{sub:min:target}
%%=============================================%%
We recall the inclusion $\cH_{K,\Omega} \subset \sC(\Omega)$ from our discussion on Feature~1 in Section~\ref{sec:regression}.
In other words, any (admissible) target $f \in \cH_{K,\Omega}$ must necessarily be a continuous function. 
To prove convergence of kernel regression with respect to $\|\cdot\|_K$, 
we won't require any stricter assumptions on $f \in \cH_{K,\Omega}$.

Nevertheless, this gives rise to the question whether or not there is a kernel function $K \in \pd(\Omega)$ satisfying $f \in \cH_{K,\Omega}$, on given $f \in \sC(\Omega)$.
The kernel $K(x,y) := f(x) \cdot f(y)$ is only one (trivial) example to give a positive answer for this question.

Another relevant question is the inclusion $\sC(\Omega) \subset \cH_{K,\Omega}$, 
i.e., is there a kernel $K \in \pd(\Omega)$, whose RKHS $\cH_{K,\Omega}$
contains {\em all} continuous functions on $\Omega$? 
If so, this would yield the equality $\sC(\Omega) = \cH_{K,\Omega}$. 
Just very recently, Steinwart~\cite{Steinwart2024} gave a negative answer on this important question.

%%=============================================%%
\subsection{Minimalistic Assumptions on the Sampling Points}
\label{sub:min:dense}
%%=============================================%%
We require that the monotonically decreasing sequence $(h_n)_{n \in \N}$ 
of fill distances in~(\ref{bas:equ:fil}) is a zero sequence, 
which is a {\em necessary} condition for the convergence of kernel regression.

In fact, if $(h_n)_{n \in \N}$ is not a zero sequence, then there must be one $h_0 > 0$ satisfying $h_n \geq h_0$ for all $n \in \N$. 
But this implies that there is one open ball $B(y,h_0) \subset \R^d$ 
centered at $y \in \Omega$ with radius $h_0>0$ which does not contain any 
point from the sequence $X = (X_k)_{k \in \N}$.
Now let $f \in \cH_{K,\Omega}$ be compactly supported with ${\rm supp}(f) \subset B(y,h_0)$ and $f \not\equiv 0$.
In this case, we have $f_{X_n} = 0$, which implies $s_n = s_{f,X_n} \equiv 0$, 
and so $\| s_n - f \|_K = \| f \|_K > 0$, for all $n \in \N$, i.e., the sequence of kernel regressions $(s_n)_{n \in \N}$ cannot convergence to $f$.

%%=============================================%%
\section{Convergence of Kernel Regression}
\label{sec:convergence}
%%=============================================%%
Now let us analyze the asymptotic behaviour of the kernel regression 
errors $(\eta_n)_{n \in \N}$ in~\eqref{bas:equ:err} for the RKHS norm $\|\cdot\|_K$
and for the maximum norm $\|\cdot\|_\infty$, respectively.
To this end, we rely on our previous work~\cite[Section~8.4.2]{Iske2018}.
For more recent results concerning the convergence of 
generalized kernel-based interpolation schemes, we refer to~\cite{AI2025}.

%%-------------------------------------------------------------------------------%%
\subsection{Convergence with respect to the RKHS Norm}
\label{sub:con:rhks}
%%-------------------------------------------------------------------------------%%

The following result (cf.~\cite[Theorem~8.37]{Iske2018})
relies on {\em minimalistic} assumptions on 
the sampling points $(x_n)_{n \in \N}$ and on the kernel $K \in \pd(\Omega)$,
as they were stated in Section~\ref{sec:minimalistic}.

\medskip

\begin{theorem}
\label{thm:convK}
Let $X = (x_n)_{n \in \N}$ be a sequence of pairwise distinct points,
whose associated fill distances $(h_n)_{n \in \N}$ in~\eqref{bas:equ:fil}
are a zero sequence, i.e., $h_n \searrow 0$ for $n \to \infty$.
Then, for any $f \in \cH_{K,\Omega}$ we have the convergence
$$  
    \eta_K(f,\cS_{K,X_n}) =
    \| s_n - f \|_K \longrightarrow 0 \quad \mbox{ for } n \to \infty.
$$
\end{theorem}

\begin{proof}
Let $y \in \Omega$.
By our assumption on $X$, there is a subsequence $(x_{n_k})_{k \in \N}$ of sampling points 
$x_{n_k} \in \Omega$ satisfying
$
  \| y - x_{n_k} \|_2 \leq h_{n_k} \longrightarrow 0 
$
for $k \to \infty$.
This implies
\begin{eqnarray*}
    \eta^2_K(K_y,\cS_{K,X_{n_k}}) 
    &\leq&
    \| K_{x_{n_k}} - K_y \|^2_K 
    =
    K(x_{n_k},x_{n_k}) - 2 K(y,x_{n_k}) + K(y,y) 
    \\ 
    & \longrightarrow & 0
    \quad \mbox{ for } k \to \infty,
\end{eqnarray*}
%
%for $k \to \infty$, 
due to the continuity of $K \in \pd(\Omega)$ on $\Omega \times \Omega$.

\medskip
Now, for a finite sequence $Y = (y_1,\ldots,y_m) \in \Omega^m$ of pairwise distinct points in $\Omega$,
we regard the function
$$
    f_{c,Y} := \sum_{j=1}^m c_j K_{y_j} \in \cS_{K,Y} \subset \cH_{K,\Omega} 
$$ 
whose coefficient vector is $c = (c_1,\ldots,c_m)^\top \in \R^m$. 

For any $y_j$, $1 \leq j \leq m$, there is a subsequence $( x_n^{(j)} )_{n \in \N}$ in $X$ satisfying $\| y_j - x_n^{(j)} \|_2 \leq h_n$.
Then, the sequence $(s_{c,n})_{n \in \N}$ of kernel regressions
$$
    s_{c,n} := \sum_{j=1}^m c_j K_{x_n^{(j)}}
    \qquad \mbox{ for } n \in \N
$$
converges to $f_{c,Y}$, i.e., $s_{c,n} \longrightarrow f_{c,Y}$, for $n \to \infty$, by
\begin{eqnarray*}
\| s_{c,n} - f_{c,Y} \|_K &=&     
     \left\| \sum_{j=1}^m c_j \left( K_{x_n^{(j)}} - K_{y_j} \right)  \right\|_K 
%    \\ & \leq &
\leq
    \sum_{j=1}^m |c_j| \cdot \| K_{x_n^{(j)}} - K_{y_j} \|_K 
    \\ &\longrightarrow & 0
    \quad \mbox{ for } n \to \infty.
\end{eqnarray*}
Thereby, kernel regression converges on the dense subset
$$
    \cS_{K,\Omega} := \left\{  f_{c,Y} \in \cS_{K,Y} \st |Y| < \infty \right\} \subset \cH_{K,\Omega},
$$
and so, as stated, also on $\cH_{K,\Omega}$ by continuous extension.
\end{proof}

We remark that the convergence of Theorem~\ref{thm:convK} may be arbitrarily slow.
Indeed, for any monotonically decreasing zero sequence $(\eta_n)_{n \in \N}$ of non-negative numbers, i.e., $\eta_n \searrow 0$,
there is a point sequence $X=(x_k)_{k \in \N}$ in $\Omega$ satisfying $h_n \searrow 0$,
and $f \in \cH_{K,\Omega}$ satisfying $\eta_K(f,\cS_{K,X_n}) \geq \eta_n$ for large enough $n \in \N$.
Since this is immaterial here, we omit further details.

%%-------------------------------------------------------------------------------%%
\subsection{Uniform Convergence}
\label{sub:con:uni}
%%-------------------------------------------------------------------------------%%
Now we analyze the convergence of kernel regression with respect to the maxi\-mum norm $\|\cdot\|_\infty$.
Recall the inclusion $\cH_{K,\Omega} \subset \sC(\Omega)$, whereby $\| \cdot \|_\infty$ is well-defined on~$\cH_{K,\Omega}$. 
We further remark that $\|\cdot\|_\infty$ is {\em weaker} than the RKHS norm $\|\cdot\|_K$, provided that $K \in \pd(\Omega)$ is bounded on $\Omega \times \Omega$.
This is due to the reproducing kernel property in~(\ref{reg:equ:rep}), whereby
$$
    | (s_n - f)(x) |^2 = 
    | (s_n - f, K_x )_K |^2 \leq
     \| s_n - f \|_K^2 \cdot \| K_x \|_K^2 = 
    \| s_n - f \|_K^2 \cdot K(x,x)
$$
for all $x \in \Omega$, which in turn implies
$$
    \| s_n - f \|_\infty \leq \| s_n - f \|_K \cdot \| \sqrt{K} \|_\infty
    \qquad \mbox{ for all } f \in \cH_{K,\Omega}.
$$

To prove uniform convergence of kernel regression, 
we rely on {\em local} $\alpha$-H\"older continuity for $K \in \pd(\Omega)$, where $\alpha>0$.

\begin{definition}
\label{bas:def:holK}
For $\Omega \subset \R^d$, let $K \in \pd(\Omega)$.
Then, $K$ is said to be {\em locally $\alpha$-H\"older continuous} on $\Omega$, for $\alpha>0$,
if every function $K_x$, for $x \in \Omega$, is locally $\alpha$-H\"older continuous on $\Omega$, i.e., 
for any $x \in \Omega$ we have
\begin{equation}
\label{con:equ:holK}
    | K_x(y_1) - K_x(y_2) | \leq C \| y_1-y_2 \|_2^\alpha
\end{equation}
for all $y_1,y_2 \in \Omega$ satisfying $\| y_1 - y_2 \|_2 < r$, for small enough $r>0$, and for some $C>0$.
For $\alpha=1$, we say that $K$ is {\em locally Lipschitz continuous} on $\Omega$.
\end{definition}

Before we continue our error analysis on kernel regression, let us first remark two 
relevant properties of locally $\alpha$-H\"older continuous kernels $K \in \pd(\Omega)$.

\begin{remark}
A translation invariant kernel $K \in \pd(\Omega)$ of the form~(\ref{bas:equ:tri}), 
i.e., $K(x,y) = \Phi(x-y)$ for $x,y \in \Omega$,
is locally $\alpha$-H\"older continuous on $\Omega$, 
iff the function $\Phi$ is locally $\alpha$-H\"older 
continuous by satisfying the growth condition 
\begin{equation*}
%\label{con:equ:holP}
    |\Phi(z_1)-\Phi(z_2)|\leq C \| z_1 - z_2 \|_2^\alpha 
    \qquad \mbox{ for all } z_1,z_2 \in \Omega \mbox{ with } \| z_1 - z_2 \|_2 < r
\end{equation*}
for small enough $r>0$ and some $C>0$.
\end{remark}

\begin{remark}
A positive definite kernel $K \in \pd(\Omega)$ on open $\Omega \subset \R^d$ 
can only be locally $\alpha$-H\"older continuous for $\alpha \leq 1$.
Indeed, for $\alpha>1$, and for any (fixed) $x \in \Omega$, the local estimate
$$
    \frac{| K_x(y_1) - K_x(y_2) |}{ \| y_1 - y_2 \|_2 } \leq C \| y_1 - y_2 \|_2^{\alpha-1}
$$
holds by (\ref{con:equ:holK}),
for all $y_1,y_2 \in \Omega$, $y_1 \neq y_2$, satisfying $\| y_1 - y_2 \|_2 < r$
for small enough $r>0$ and for some $C>0$. 
But this means that all directional derivatives of $K_x$ 
must vanish at all points in $\Omega$, due to the mean value theorem.
In this case, $K$ is constant on $\Omega \times \Omega$, so that
$K$ cannot be positive definite, i.e., $K \not\in \pd(\Omega)$.
\end{remark}

From now on, we assume that $K \in \pd(\Omega)$ is locally $\alpha$-H\"older continuous for $\alpha \in (0,1]$.
Note that this condition on $K$ is slightly more restrictive than the {\em minimalistic} assumption of continuity 
for $K$ on $\Omega \times \Omega$ in Section~\ref{sub:min:kernel}.

\medskip
Now we show that all functions in the RKHS $\cH_{K,\Omega}$ are {\em locally} $\alpha/2$-H\"older continuous, 
if $K \in \pd(\Omega)$ is locally $\alpha$-H\"older continuous on $\Omega$.

%\medskip
\begin{lemma}
\label{con:lem:holf}
For $\Omega \subset \R^d$, let $K \in \pd(\Omega)$ be locally $\alpha$-H\"older continuous on $\Omega$, for some $\alpha \in (0,1]$.
Then, all functions in $\cH_{K,\Omega}$ are locally $\alpha/2$-H\"older continuous on $\Omega$.
\end{lemma}

\begin{proof}
Let $f \in \cH_{K,\Omega}$ and $x \in \Omega$ be fixed.
Then, we have
\begin{eqnarray*}
    |f(x)-f(y)|^2 &=& |(K_x - K_y,f)_K|^2
    \leq \| K_x - K_y \|_K^2 \cdot \| f \|_K^2    
    \\ & = & 
    \left( K_x(x) - K_x(y) + K_y(y) - K_y(x) \right) \cdot \| f \|_K^2 
    \\ & \leq & 
    2 C \| x-y \|_2^{\alpha} \cdot \| f \|_K^2
\end{eqnarray*}
for some $C>0$, 
and where $y \in \Omega$, $x \neq y$, is required to satisfy $\|x-y\|_2 < r$, for $r>0$ small enough.
\end{proof}

\medskip
From Lemma~\ref{con:lem:holf}, we can directly conclude the following error estimate
for kernel regression from finite sampling points.

\begin{proposition}
\label{con:pro:uni}
For $\alpha \in (0,1]$, let $K \in \pd(\Omega)$ be locally $\alpha$-H\"older continuous on $\Omega \subset \R^d$.
Moreover, let $X \subset \Omega$ be a finite subset of $\Omega$.
Then, we have for any $f \in \cH_{K,\Omega}$ the error estimate
$$
    \| s_{f,X} - f \|_\infty \leq \sqrt{2 C h^\alpha_{X,\Omega}} \cdot \| f \|_K,
$$
where $s_{f,X}$ denotes the interpolant to $f$ on $X$. 
\end{proposition}

\begin{proof}
Suppose $y \in \Omega$. Then, there is one $x \in X$ satisfying $\| y - x \|_2 \leq h_{X,\Omega}$. 
By using
$
   (s_{f,X} - f)(x)=0
$
we can conclude
\begin{eqnarray*}
    |(s_{f,X}-f)(y)|^2 &=& |(s_{f,X}-f)(y) - (s_{f,X}-f)(x)|^2    
    \\ &\leq& 
    2 C  h^\alpha_{X,\Omega} \cdot \| s_{f,X} - f\|_K^2
    \\ &\leq&
    2 C  h^\alpha_{X,\Omega} \cdot \| f\|_K^2
\end{eqnarray*}
from Lemma~\ref{con:lem:holf},
where we further used
$$
    \| s_{f,X} - f\|_K \leq \|f\|_K,
$$
from the stability estimates in~(\ref{reg:equ:nmin}).
\end{proof}

\medskip
Finally, our next result follows directly from Proposition~\ref{con:pro:uni}.

\medskip
\begin{corollary}
\label{con:cor:hol}
For $\alpha \in (0,1]$, let $K \in \pd(\Omega)$ be locally $\alpha$-H\"older continuous on $\Omega \subset \R^d$.
Moreover, let $X = (x_k)_{k \in \N}$ be a sequence of pairwise distinct points in $\Omega$, 
whose corres\-ponding sequence $(h_n)_{n \in \N}$ of fill distances $h_n = h(X_n,\Omega)$, as in~\eqref{bas:equ:fil}, 
is a zero sequence, i.e.,~$h_n \searrow 0$, for $n \to \infty$.
Then, the uniform convergence
\begin{equation*}
%\label{con:equ:hol}
    \|s_n - f\|_\infty = \cO \left( h_n^{\alpha/2} \right) \quad \mbox{ for } n \to \infty.
\end{equation*}
holds for all $f \in \cH_{K,\Omega}$ at convergence rate $\alpha/2$.
\end{corollary}

%%=============================================%%
\section{Conclusion and Future Work}
\label{sec:conclusion}
%%=============================================%%
We have proven convergence of kernel regression from irregular samples
in reproducing kernel Hilbert spaces (RKHS), under minimalistic assumptions
(cf.~Section~\ref{sec:minimalistic})
on the kernel $K$, its RKHS $\cH_{K,\Omega}$, and the sampling points $X \subset \Omega$.

\medskip
Now it may be inspiring to work on even weaker conditions for $K$ and $X$, under which kernel regression is convergent.

\medskip
Yet, it remains to analyse conditions for $K \in \pd(\Omega)$, under which given functions $f \in \sC(\Omega)$ lie in $\cH_{K,\Omega}$, 
so that kernel regression converges to $f$ (due to Theorem~\ref{thm:convK}).
And if so, i.e., if $f \in \cH_{K,\Omega}$, can we then conclude properties of $K$ from properties of $f$?
E.g.~if $f \in \cH_{K,\Omega}$ is (locally) $\alpha$-H\"older continuous, for $\alpha \in (0,1]$, can we then conclude that $K$ is
also (locally) $\alpha$-H\"older continuous, 
so that $f$ can be approximated at convergence rate $\alpha/2$ (due to Corollary~\ref{con:cor:hol})?

%%%%%%%%%%%%%%%%%%%%%%%%%%%%%%%%%%%
% R E F E R E N C E S
%%%%%%%%%%%%%%%%%%%%%%%%%%%%%%%%%%%
% Non-BibTeX users please use


\begin{thebibliography}{11}
%
\bibitem{AI2025}
K.~Albrecht and A.~Iske:
On the convergence of generalized kernel-based interpolation by greedy data selection algorithms. 
{\sl BIT Numer.\ Math.}~{\bf 65}, 5 (2025). https://doi.org/10.1007/s10543-024-01048-3.
%
\bibitem{Aronszajn1950}
N.~Aronszajn:
Theory of reproducing kernels.
{\sl Transactions of the American Mathematical Society}~{\bf 68} (3), 1950, 337--404.
%
\bibitem{Buhmann2003}
M.D.~Buhmann:
{\sl Radial Basis Functions}.
Cambridge University Press, Cambridge, UK, 2003.
%
\bibitem{Cucker2001}
F.~Cucker and S.~Smale:
On the mathematical foundations of learning.
{\sl Bull.\ Amer.\ Math.\ Soc.}~{\bf 39} 2001, 1--49,.
%
\bibitem{Cucker2007}
F.~Cucker and D.-X.~Zhou:
{\sl Learning Theory: An Approximation Theory Viewpoint}. 
Cambridge University Press, 2007.
%
\bibitem{EIT2023}
S.~Eckstein, A.~Iske, and M.~Trabs:
Dimensionality reduction and Wasserstein stability for kernel regression.
{\sl Journal of Machine Learning Research}~{\bf 24} (334), 2023, 1--35.
%
\bibitem{Fukumizu2009}
K.~Fukumizu, F.~Bach, and M.~Jordan:
Kernel dimension reduction in regression.
{\sl Annals of Statistics}~{\bf 37}, 2009, 1871--1905.
%
\bibitem{Iske2018}
A.~Iske:
{\sl Approximation Theory and Algorithms for Data Analysis}.
Texts in Applied Mathematics, vol.~68, Springer, Cham, 2018.
%
\bibitem{Schoelkopf2002}
B.~Sch\"olkopf and A.J.~Smola: 
{\sl Learning with Kernels}. 
MIT Press, Cambridge, 2002.
%
\bibitem{Smale2003}
S.~Smale and D.X.~Zhou:
Estimating the approximation error in learning theory.
{\sl Anal.\ Appl.}~{\bf 1}, 2003, 17--41.
%
\bibitem{Steinwart2008}
I.~Steinwart and A.~Christmann: 
{\sl Support Vector Machines}. 
Springer, New York, 2008.
%
\bibitem{Steinwart2024}
I.~Steinwart:
Reproducing kernel Hilbert spaces cannot contain all continuous functions on a compact metric space.
{\sl Archiv der Mathematik}~{\bf 122}, 2024, 553--557. 
%
\bibitem{Vapnik1998}
V.~Vapnik: 
{\sl Statistical Learning Theory}.
Wiley, New York, 1998.
%
\bibitem{Wendland2005}
H.~Wendland:
{\sl Scattered Data Approximation}.
Cambridge University Press, Cambridge, UK, 2005.
%
\bibitem{Zhou2003}
D.-X.~Zhou:
Capacity of reproducing kernel spaces in learning theory.
{\sl IEEE Transactions on Information Theory}~{\bf 49} (7), July~2003, 1743--1752.
%
\end{thebibliography}
\end{document}